%% file: content.tex
\documentclass[letterpaper, 10 pt, conference]{ieeeconf}  % Comment this line out if you need a4paper

\IEEEoverridecommandlockouts                              % This command is only needed if 
\input{preamble}

\makeatother
\DeclareCaptionFont{mysize}{\fontsize{8}{9.6}\selectfont}
\title{\LARGE \bf Backup-Based Safety Filters: A Comparative Review of Backup CBF, Model Predictive Shielding, and \gatekeeper{}}

\author{Taekyung Kim$^{1}$, Aswin D. Menon$^{2}$, Akshunn Trivedi$^{1}$, Dimitra Panagou$^{1,2}$% <-this % stops a space
\thanks{This research was supported in part by the National Science Foundation (NSF) under Award Number 1942907.} %
\thanks{$^{1}$Department of Robotics, $^{2}$Department of Aerospace Engineering, University of Michigan, Ann Arbor, MI, 48109, USA {\tt\footnotesize \{taekyung, admenon, akshunn, dpanagou\}@umich.edu} } %
}

\begin{document}
\maketitle
\thispagestyle{empty}
\pagestyle{empty}

%%%%%%%%%%%%%%%%%%%%%%%%%%%%%%%%%%%%%%%%%%%%%%%%%%%%%%%%%%%%%%%%%%%%%%%%%%%%%%%%
% text friendly version:
%This paper revisits three backup-based safety filters—Backup Control Barrier Functions (Backup CBF), Model Predictive Shielding (MPS), and gatekeeper—through a unified comparative framework. Using a common safety-filter abstraction and shared notation, we make explicit both their common backup-policy structure and their key algorithmic differences. We compare the three methods through their filter-inactive sets, i.e., the states where the nominal policy is left unchanged. In particular, we show that MPS is a special case of gatekeeper, and we further relate gatekeeper to the interior of the Backup CBF inactive set within the implicit safe set. This unified view also highlights a key source of conservatism in backup-based safety filters: safety is often evaluated through the feasibility of a backup maneuver, rather than through the nominal policy's continued safe execution. The paper is intended as a compact tutorial and review that clarifies the theoretical connections and differences among these methods.

\begin{abstract}
This paper revisits three backup-based safety filters---Backup Control Barrier Functions (Backup CBF), Model Predictive Shielding (MPS), and \gatekeeper{}---through a unified comparative framework. Using a common safety-filter abstraction and shared notation, we make explicit both their common backup-policy structure and their key algorithmic differences. We compare the three methods through their filter-inactive sets, i.e., the states where the nominal policy is left unchanged. In particular, we show that MPS is a special case of \gatekeeper{}, and we further relate \gatekeeper{} to the interior of the Backup CBF inactive set within the implicit safe set. This unified view also highlights a key source of conservatism in backup-based safety filters: safety is often evaluated through the feasibility of a backup maneuver, rather than through the nominal policy's continued safe execution. The paper is intended as a compact tutorial and review that clarifies the theoretical connections and differences among these methods. \href{https://www.taekyung.me/backup-safety-filters}{\textcolor{red}{[Project Page]}}\footnote{Project page: \href{https://www.taekyung.me/backup-safety-filters}{https://www.taekyung.me/backup-safety-filters}} \href{https://github.com/tkkim-robot/safe_control/tree/main/shielding}{\textcolor{red}{[Code]}} \href{https://backupsf.taekyung.me/}{\textcolor{red}{[Web Demo]}}
\end{abstract}
%%%%%%%%%%%%%%%%%%%%%%%%%%%%%%%%%%%%%%%%%%%%%%%%%%%%%%%%%%%%%%%%%%%%%%%%%%%%%%%%

\section{INTRODUCTION}

Ensuring safety in autonomous systems with nonlinear dynamics remains a central challenge in robotics and cyber-physical systems. Modern planning and control modules, including reinforcement learning~(RL)~\cite{thananjeyan_recovery_2021, ibarz_how_2021} and sampling-based motion planning~\cite{williams_informationtheoretic_2018, kim_smooth_2022}, often fail to provide formal safety guarantees. A common remedy is to place a \emph{safety filter} between the nominal decision-making algorithms and the plant~\cite{hsu_safety_2024}. The safety filter monitors the nominal command and modifies it only when necessary to preserve safety. Ideally, this intervention is minimally invasive: the nominal controller should remain active whenever its use is certifiably compatible with future safe operation.

Historically, one route to such guarantees has been reachability analysis, which, although rigorous, is computationally prohibitive for high-dimensional nonlinear systems in real time~\cite{bansal_hamiltonjacobi_2017}. Control Barrier Functions~(CBFs) provide a tractable alternative by encoding a safety condition that is imposed pointwise in state and time through an optimization problem, typically formulated as a quadratic program~(QP)~\cite{ames_control_2019}. The main difficulty, however, is constructing a CBF that enforces invariance for nonlinear systems with input constraints~\cite{garg_advances_2024}; a geometric safety constraint alone need not guarantee the existence of a control input that prevents future constraint violation~\cite{kim_learning_2025, kim_how_2025}.

This observation motivates \emph{backup-based} safety filters~\cite{hsu_safety_2024}. Rather than beginning with an explicit analytical barrier function, these methods begin with a known backup maneuver, referred to here as a \emph{backup policy}. Safety is certified by showing that, from the current state or from a predicted future state, the system can switch to that backup policy, remain inside the safe set, and reach a terminal controlled invariant set. Backup CBF~\cite{chen_backup_2021}, Model Predictive Shielding~(MPS)~\cite{bastani_safe_2021, bastani_safe_2021a}, and \gatekeeper{}~\cite{agrawal_gatekeeper_2023, agrawal_gatekeeper_2024} all follow this general principle.

\begin{table}[t]
\centering
\footnotesize
\caption{Comparison overview of backup-based safety filters}
\label{tab:math_comparison}
\begin{tabular}{p{1.15cm} p{2.0cm} p{1.8cm} p{1.8cm}}
\toprule
 & \multicolumn{1}{c}{\textbf{Backup CBF}} & \multicolumn{1}{c}{\textbf{MPS}} & \multicolumn{1}{c}{\textbf{\gatekeeper{}}} \\
\midrule
\textbf{Prerequisite} & \multicolumn{3}{c}{Backup policy $\pi_{\textup{b}}$} \\
\midrule
\multirow{2}{*}{\shortstack[l]{\textbf{Decision}\\ \textbf{Variable}}} & Control input & Boolean & Switching time \\
 & $\vu \in \ControlSpace$ & acceptance & $T_S \in [0, T_H]$ \\
\midrule
\multirow{2}{*}{\shortstack[l]{\textbf{Verification}\\ \textbf{Mechanism}}} & Enforce \eqref{eq:backup_cbf_qp_constraint_path}-\eqref{eq:backup_cbf_qp_constraint_terminal} & Checks validity & Checks validity \\
&  $\forall \tau \in [t, t+T_B]$ & at $T_S = \Delta t$ & for $\max T_S$ \\
\midrule
\multirow{2}{*}{\shortstack[l]{\textbf{Limitation}}} & \multicolumn{2}{c}{Safety evaluation on backup} & Dependency \\
 & \multicolumn{2}{c}{(\autoref{subsec:evaluation_on_backup}) } & on $T_H$ \\
\midrule
\multirow{2}{*}{\shortstack[l]{\textbf{Nominal}\\ \textbf{Tracking}}} & Modifies $\vu_{\textup{nom}}$ & Binary accept- & Maximizes tra- \\
 & via QP \eqref{eq:backup_cbf_qp} & ance on $\vu_{\textup{nom}}$ & cking duration \\
\bottomrule
\end{tabular}
\end{table}

\begin{figure}[t]
    \centering
\includegraphics[width=0.99\linewidth]{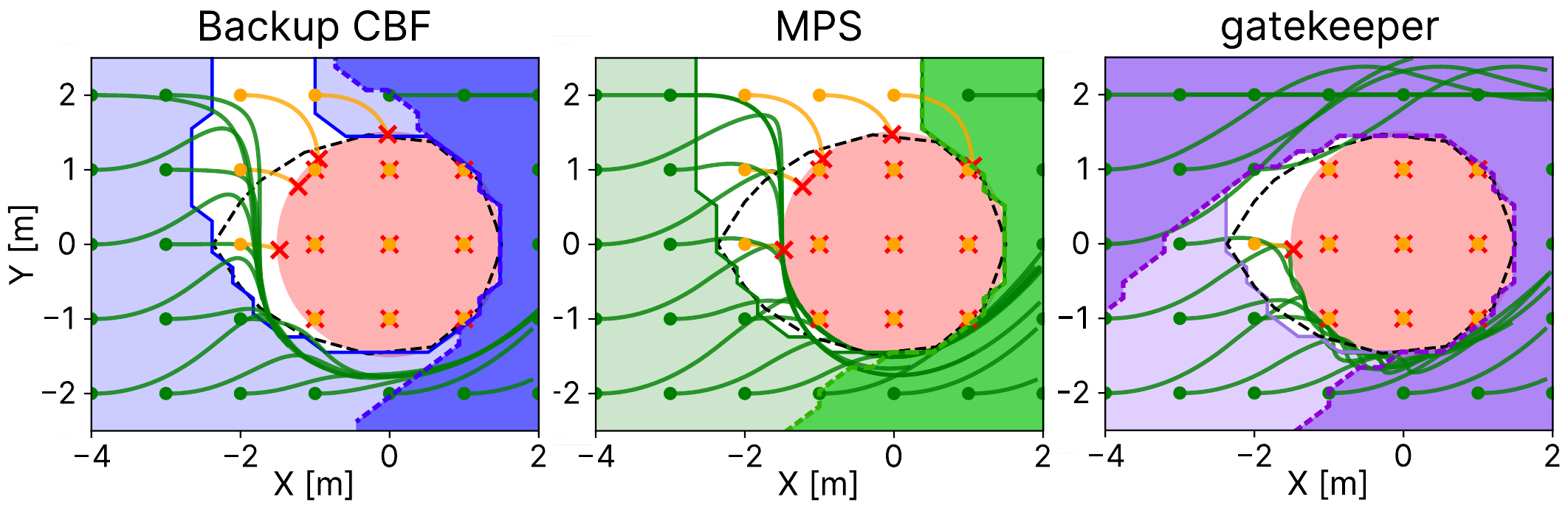}
    \caption{Recovered \textbf{safe sets} (light-colored regions) and \textbf{filter-inactive sets} (dark-colored regions) for the double-integrator example on the slice $(v_x,v_y)=(2,0)$. The black dashed curve shows the viability kernel obtained from HJ reachability on the full 4-dimensional state space. For each $1~\textup{m}\times 1~\textup{m}$ grid of initial positions, we simulate the closed-loop system under each safety filter; green trajectories remain safe, whereas yellow trajectories indicate unsafe or infeasible rollouts.}
\label{fig:double_integrator_sets}
\end{figure}

The contribution of this paper is expository and comparative. We revisit Backup CBF, MPS, and \gatekeeper{} within a unified framework that makes their common structure and key differences explicit (see \autoref{tab:math_comparison}). First, we define a common safety-filter abstraction and trajectory notation and restate the three methods using consistent notation. Second, we raise a structural limitation in backup-based safety filters, which we call \emph{safety evaluation on backup}: safety is often evaluated through the feasibility of a backup maneuver, rather than through the nominal policy's continued safe execution. By comparing the three methods through their filter-inactive sets, we show that \gatekeeper{} can mitigate this conservatism by optimizing over the switching time. Third, we demonstrate these differences through simulation studies that support the theoretical claims, including a planar double-integrator example (see Fig.~\ref{fig:double_integrator_sets}) and two dynamic obstacle avoidance scenarios. This paper is intended as a compact review and tutorial, not as a new algorithmic proposal.

\section{PROBLEM FORMULATION \label{sec:problem}}
\subsection{Mathematical Setup}

We consider a continuous-time control-affine system
\begin{equation}
\dot{\vx} = f(\vx) + g(\vx)\vu,
\label{eq:system}
\end{equation}
where $\vx \in \StateSpace \subseteq \Rn$ is the state and $\vu \in \ControlSpace \subset \Rm$ is the control input, with $\ControlSpace$ denoting the admissible input set. $f: \StateSpace \to \Rn$ and $g: \StateSpace \to \Rnm$ are locally Lipschitz. We assume a nominal policy 
$\pi_{\textup{nom}}: \StateSpace \to \ControlSpace$ that is task-oriented but not necessarily safe. The safe set and terminal set used throughout the paper are defined as:
\begin{equation}
\calC \coloneqq \{\vx \in \StateSpace \mid h^{\calC}(\vx) \ge 0\},
\label{eq:bcbf_safeset}
\end{equation}
\begin{equation}
\calS_0 \coloneqq \{\vx \in \StateSpace \mid h^{\calS}(\vx) \ge 0\} \subseteq \calC,
\label{eq:bcbf_CI_set}
\end{equation}
where $h^{\calC}, h^{\calS}: \StateSpace \to \RealSpace$ are continuously differentiable.

For any locally Lipschitz feedback policy (or controller) $\pi:\StateSpace\to\ControlSpace$, let $\varphi_t^{\pi}(\vx_0)$ denote the state at time $t\ge 0$ of the closed-loop system $\dot{\vx}=f(\vx)+g(\vx)\pi(\vx)$ with initial condition $\vx(0)=\vx_0$. 
A set $\calA \subseteq \StateSpace$ is \emph{forward invariant} under $\pi$ if $\varphi_t^{\pi}(\vx_0)\in \calA$ for every $\vx_0\in\calA$ and all $t\ge 0$. The set $\calA$ is \emph{controlled invariant} if there exists a feedback policy that renders $\calA$ forward invariant. We assume that $\calS_0$ is a known controlled invariant terminal
set. The additional recovery requirement imposed on a backup
policy is stated in \autoref{def:backup_policy}.

We next introduce two definitions that provide the common language used throughout the paper. We first define the recoverable set induced by a general feedback policy.
\begin{definition}[Recoverable set induced by a feedback policy]\label{def:recoverable_set}
Let $\pi:\StateSpace\to\ControlSpace$ be a feedback policy, let $\calS_0 \subseteq \calC \subseteq \StateSpace$,
and let $T>0$. The \emph{recoverable set} of $\calS_0$ under $\pi$ over horizon $T$ is
\begin{equation}\label{eq:recoverable_set_general}
\begin{aligned}
\calR^{\pi}(T;\calS_0)
&\coloneqq \{ \vx_0 \in \calC \mid  \varphi_t^{\pi}(\vx_0) \in \calC,\ \forall t \in [0,T], \\
& \qquad \qquad \qquad \quad \varphi_T^{\pi}(\vx_0) \in \calS_0 \}.
\end{aligned}
\end{equation}
Thus, $\calR^{\pi}(T;\calS_0)$ is the set of all initial states from which the policy $\pi$
keeps the system inside $\calC$ over $[0,T]$ and reaches $\calS_0$ at time $T$.
\end{definition}

We define a backup policy based on the recoverable set.
\begin{definition}[Backup policy]\label{def:backup_policy}
Let $\calS_0 \subseteq \calC \subseteq \StateSpace$, and let
$T_B>0$ be a prescribed backup horizon. A locally Lipschitz
feedback policy
$\pi_{\textup{b}}:\StateSpace\to\ControlSpace$ is called a \emph{backup policy} for $\calS_0$
if the following two properties hold:

\textup{(i)} The recoverable set induced by $\pi_{\textup{b}}$ at horizon $T_B$, denoted $\calR^{\pi_{\textup{b}}}(T_B;\calS_0)$, contains a neighborhood of $\calS_0$.

\textup{(ii)} The set $\calS_0$ is forward invariant under $\pi_{\textup{b}}$. 
\end{definition}

In words, a backup policy keeps the system inside $\calS_0$ once it is there, and can also drive all states in a surrounding recoverable region~$\calR^{\pi_{\textup{b}}}(T_B;\calS_0)$ into $\calS_0$ within the prescribed backup horizon $T_B$ without leaving the safe set $\calC$.

The three methods use this recoverable set differently. Backup
CBF uses $\calR^{\pi_{\textup{b}}}(T_B;\calS_0)$ as an implicit
safe set. MPS and \gatekeeper{} instead use backup rollouts to
validate nominal-then-backup candidate trajectories and execute
the backup policy when required. We make this difference explicit in the subsequent sections.

\subsection{Problem Statement}

We now introduce a unified notion of \textbf{\emph{(backup-based) safety filter}}, and then state the formal comparison problem studied in this paper.

\begin{definition}[Safety filter]\label{def:safety_filter}
Given nominal and backup policies $\pi_{\textup{nom}}$ and
$\pi_{\textup{b}}$, sets $\calS_0\subseteq\calC$, and
algorithm-specific parameters $\Theta_{\textup{sf}}$, a policy
$\pi_{\textup{sf}}:\StateSpace\to\ControlSpace$ is called a
\emph{safety filter} for System~\eqref{eq:system} on a prescribed
set of initial states $\mathcal D\subseteq\calC$ if, for every
$\vx\in\mathcal D$, the mapping
$\vu^\star =
\pi_{\textup{sf}}\big(
\vx;\pi_{\textup{nom}},\pi_{\textup{b}},
\calC,\calS_0,\Theta_{\textup{sf}}\big)\in\ControlSpace$
results in
$\varphi_t^{\pi_{\textup{sf}}}(\vx)\in\calC$
for all $t\ge0$.
\end{definition}

For digital filters, $\pi_{\textup{sf}}$ denotes the decision
when a new validity check is performed, and
$\varphi_t^{\pi_{\textup{sf}}}$ denotes the closed-loop trajectory
under the corresponding update rule.

\begin{definition}[Filter-inactive set]\label{def:filter_inactive_set}
Let $\pi_{\textup{sf}}$ be a safety filter. Its \emph{filter-inactive set} is
\begin{equation}
\calI_{\textup{sf}}
\coloneqq
\left\{
\vx \in \calC \ \middle| \
\pi_{\textup{sf}}\big(\vx;\pi_{\textup{nom}},\pi_{\textup{b}},\calC,\calS_0,\Theta_{\textup{sf}}\big)
=
\pi_{\textup{nom}}(\vx)
\right\}.
\label{eq:filter_inactive_set}
\end{equation}
Thus, $\vx \in \calI_{\textup{sf}}$ if and only if the safety filter leaves the nominal control input unchanged at that query state $\vx$.
\end{definition}

In this paper, we focus on the three safety filters induced by Backup CBF, MPS, and
\gatekeeper{}, denoted by $\pi_{\textup{BCBF}}$, $\pi_{\textup{MPS}}$, and $\pi_{\textup{GK}}$,
respectively. Our goal is to compare these three filters within the common abstraction of
\autoref{def:safety_filter}, with particular emphasis on the set of states where each filter leaves the nominal control input unchanged. Here, $\calI_{\textup{BCBF}}$, $\calI_{\textup{MPS}}$, and $\calI_{\textup{GK}}$ denote the filter-inactive sets of $\pi_{\textup{BCBF}}$, $\pi_{\textup{MPS}}$, and $\pi_{\textup{GK}}$, respectively.

\begin{problem}\label{prob:comparison_filter_inactive_sets}
Given the three safety filters $\pi_{\textup{BCBF}}, \pi_{\textup{MPS}}$, and $\pi_{\textup{GK}}$, characterize the set relationships among their filter-inactive sets. In particular, establish the inclusion $\calI_{\textup{MPS}} \subseteq \calI_{\textup{GK}}$ (see \autoref{thm:mps_special_case}), and characterize the relation between $\calI_{\textup{GK}}$ and $\operatorname{int}_{\calS}(\calI_{\textup{BCBF}})$, the relative interior of the Backup CBF inactive set within the implicit safe set $\calS$ (see \autoref{thm:backup_cbf_interior}).
\end{problem}

A larger filter-inactive set means that the safety filter intervenes less often on the nominal policy while preserving its own safety guarantee. The remainder of the paper reviews Backup CBF, MPS, and \gatekeeper{} within the unified framework above, and then proves the set inclusions posed in \autoref{prob:comparison_filter_inactive_sets}.

\section{REVIEW OF BACKUP CBF \label{sec:backup_cbf}}

Backup CBF~\cite{chen_backup_2021} $\pi_{\textup{BCBF}}$ defines the implicit safe set as the recoverable set induced by the backup policy, and then solves a QP to keep the state inside that set. In contrast to MPS and \gatekeeper{}, the output of Backup CBF online is not a binary choice between $\pi_{\textup{nom}}$ and $\pi_{\textup{b}}$; instead, it is the closest admissible input to the nominal input that satisfies the barrier constraints.

\subsection{Implicit safe set induced by the backup policy}

Given the backup policy $\pi_{\textup{b}}$ and a backup horizon $T_B>0$, Backup CBF defines the implicit safe set
\begin{equation} \label{eq:implicit_set}
\begin{aligned}
\calS  \coloneqq \calR^{\pi_{\textup{b}}}(T_B;\calS_0) = & \{ \vx \in \calC  \mid \forall \tau \in [0,T_B], \\
& \varphi_{\tau}^{\pi_{\textup{b}}}(\vx) \in \calC, \ \varphi_{T_B}^{\pi_{\textup{b}}}(\vx) \in \calS_0 \}.
\end{aligned}
\end{equation}
By \autoref{def:recoverable_set}, the set $\calS$ is the recoverable set induced by the backup policy~$\pi_{\textup{b}}$ over horizon $[0, T_B]$. It consists of all states from which immediate execution of the backup policy remains inside $\calC$ and reaches $\calS_0$ after time $T_B$.

\begin{theorem}[{\cite{gurriet_online_2018}}]\label{thm:implicit_set}
Assume that $\calS_0 \subseteq \calC$ is forward invariant under the backup policy $\pi_{\textup{b}}$. Then, for every $T_B>0$, the set $\calS$ in \eqref{eq:implicit_set} is a controlled invariant set and satisfies $\calS_0 \subseteq \calS \subseteq \calC$.
\end{theorem}

In the Backup CBF construction, the feedback law is provided online by the quadratic program below. The horizon $T_B$ determines how much the terminal set $\calS_0$ is enlarged: when $T_B=0$, one recovers $\calS=\calS_0$, whereas increasing $T_B$ enlarges the implicit safe set.

\subsection{Backup CBF-QP}

The Backup CBF $h_{\textup{BCBF}}$ is defined from the minimum safety margin observed along the predicted backup trajectory:
\begin{equation}
h_{\textup{BCBF}}(\vx)
\coloneqq
\min\left\{
h^{\calS}\left(\varphi_{T_B}^{\pi_{\textup{b}}}(\vx)\right),
\min_{\tau \in [0,T_B]} h^{\calC}\left(\varphi_{\tau}^{\pi_{\textup{b}}}(\vx)\right)\right\}. \nonumber
\label{eq:backup_cbf_def}
\end{equation}
By construction, the $0$-superlevel set of $h_{\textup{BCBF}}$ coincides with the implicit safe set in \eqref{eq:implicit_set}.

To implement the Backup CBF online, $h_{\textup{BCBF}}$ is written as a function of time along the current closed-loop trajectory. At time $t$, with current state $\vx(t)$, define
\begin{equation}
\begin{aligned}
h_{\textup{BCBF}}(t)
&\coloneqq \min\Bigl\{ h^{\calS}\bigl(\varphi_{T_B}^{\pi_{\textup{b}}}(\vx(t))\bigr), \\
&\hspace{3.7em}\min_{\tau \in [t,t+T_B]} h^{\calC}\bigl(\varphi_{\tau-t}^{\pi_{\textup{b}}}(\vx(t))\bigr) \Bigr\}.
\label{eq:backup_cbf_time}
\end{aligned}
\end{equation}
The corresponding CBF constraint is
\begin{equation}\label{eq:backupcbf_original_constraint}
\dot{h}_{\textup{BCBF}}(t) + \alpha\!\left(h_{\textup{BCBF}}(t)\right) \ge 0,
\end{equation}
where $\alpha$ is a class-$\mathcal{K}$ function. Following~\cite{chen_backup_2021}, Backup CBF does not enforce \eqref{eq:backupcbf_original_constraint} directly on the minimum in \eqref{eq:backup_cbf_time}; instead, it enforces sufficient constraints on the individual terms appearing in that minimum.

Backup CBF computes the filtered input by solving the following QP:
\begin{subequations}\label{eq:backup_cbf_qp}
\begin{align}
& \quad \vu^{\star}_{\textup{BCBF}} =\ 
\pi_{\textup{BCBF}}\big(\vx;\pi_{\textup{nom}},\pi_{\textup{b}},\calC,\calS_0,T_B,\alpha\big) \label{eq:backup_cbf_qp_obj} \\
& \qquad \quad\  \coloneqq \  \argmin_{\vu \in \ControlSpace} \ \|\vu-\vu_{\textup{nom}}\|_2^2 \notag \\
\text{s.t.} \quad &  \forall \tau \in [t,t+T_B] \nonumber \\
& \frac{\operatorname{d} h^{\calC}\!\left(\varphi_{\tau-t}^{\pi_{\textup{b}}}(\vx)\right)}{\operatorname{d}t}(\vx,\vu)
+ \alpha\!\left(h^{\calC}\!\left(\varphi_{\tau-t}^{\pi_{\textup{b}}}(\vx)\right)\right) \ge 0, \label{eq:backup_cbf_qp_constraint_path} \\
& \frac{\operatorname{d} h^{\calS}\!\left(\varphi_{T_B}^{\pi_{\textup{b}}}(\vx)\right)}{\operatorname{d}t}(\vx,\vu)
+ \alpha\!\left(h^{\calS}\!\left(\varphi_{T_B}^{\pi_{\textup{b}}}(\vx)\right)\right) \ge 0, 
\label{eq:backup_cbf_qp_constraint_terminal}
\end{align}
\end{subequations}
where $\vu_{\textup{nom}} \coloneqq \pi_{\textup{nom}}(\vx)$ is the nominal input. The derivatives in \eqref{eq:backup_cbf_qp_constraint_path}-\eqref{eq:backup_cbf_qp_constraint_terminal} are taken with respect to the current system dynamics $\dot{\vx}=f(\vx)+g(\vx)\vu$ at the current state $\vx(t)$, with $\tau$ held fixed in \eqref{eq:backup_cbf_qp_constraint_path}.

Let $\calI_{\textup{BCBF}}$ denote the filter-inactive set of $\pi_{\textup{BCBF}}$ as in \autoref{def:filter_inactive_set}. Equivalently, $\vx \in \calI_{\textup{BCBF}}$ if and only if the nominal input~$\vu_{\textup{nom}}$ already satisfies all constraints in \eqref{eq:backup_cbf_qp_constraint_path}-\eqref{eq:backup_cbf_qp_constraint_terminal}.

Thus, unlike MPS and \gatekeeper{}, Backup CBF uses the backup policy only to construct the implicit safe set and the associated constraints; the control input applied online is the QP solution \eqref{eq:backup_cbf_qp}, not the backup policy itself. In implementation, the continuum of constraints in \eqref{eq:backup_cbf_qp_constraint_path} is approximated by finitely many collocation points~\cite{chen_backup_2021}.

\section{REVIEW OF MODEL PREDICTIVE SHIELDING}

Model Predictive Shielding (MPS)~\cite{bastani_safe_2021} $\pi_{\textup{MPS}}$ provides the second instantiation of \autoref{def:safety_filter}. The original development of the method is for discrete-time systems. For comparison with Backup CBF and \gatekeeper{}, we restate its logic in continuous time, with the digital update made explicit, while keeping the algorithmic mechanism unchanged. The plant state evolves continuously, but the monitor is queried only at discrete update times.

The original MPS presentation separates the backup mechanism into a recovery component and an invariance component~\cite{bastani_safe_2021}. For the purposes of the present paper, we absorb both into the single backup policy $\pi_{\textup{b}}$ introduced in \autoref{def:backup_policy}. This keeps the comparison with Backup CBF and \gatekeeper{} at the level of the shared safety-filter abstraction.

\subsection{Digital update times, candidate trajectories, and validity}

Let $\Delta t>0$ denote the control-update interval of the digital safety monitor, and define the update times by $t_k \coloneqq t_0 + k\Delta t$, $k \in \mathbb{Z}_{\ge 0}$. At each update time, the current state is $\vx_k \coloneqq \vx(t_k)$. Over each interval $[t_k,t_{k+1})$, the monitor does not recompute its decision; it continuously applies the feedback law selected at time $t_k$.

\begin{definition}[Candidate trajectory]\label{def:candidate_trajectory}
Let $\vx_k=\vx(t_k)$ be the current state at update time $t_k$, and let $T_S \ge 0$ denote a switching time. The corresponding \emph{candidate trajectory} is the map $\chi^{\textup{can}}(\tau;\vx_k,T_S):[0,T_S+T_B]\to\StateSpace$ defined by
\begin{equation} \nonumber %\label{eq:candidate_trajectory}
\chi^{\textup{can}}(\tau;\vx_k,T_S) 
\coloneqq
\left\{
\begin{array}{@{}l@{\,}l@{}}
\varphi_{\tau}^{\pi_{\textup{nom}}}(\vx_k), & \tau \in [0,T_S], \\[1mm]
\varphi_{\tau-T_S}^{\pi_{\textup{b}}}\!\left(\varphi_{T_S}^{\pi_{\textup{nom}}}(\vx_k)\right), & \tau \in [T_S,T_S+T_B].
\end{array}
\right.
\end{equation}
\end{definition}
Thus, the candidate trajectory first follows the nominal policy for duration $T_S$, and then follows the backup policy for duration $T_B$.

\begin{definition}[Validity indicator]\label{def:valid}
For a current state $\vx_k$ and horizons $T_S,T_B \ge 0$, let
$\chi_{\vx_k}^{\textup{can}}(\tau) \coloneqq \chi^{\textup{can}}(\tau;\vx_k,T_S)$. Define the validity indicator
\begin{equation}
\operatorname{Valid}(\vx_k;T_S,T_B)
\coloneqq
\left\{
\begin{array}{@{}l@{\,}l@{}}
1, & \chi_{\vx_k}^{\textup{can}}(\tau) \in \calC,\ \forall \tau \in [0,T_S+T_B], \\
& \chi_{\vx_k}^{\textup{can}}(T_S+T_B) \in \calS_0, \\
0, & \text{otherwise}. 
\end{array}
\right. \nonumber
%\label{eq:valid_indicator}
\end{equation}
Equivalently, $\operatorname{Valid}(\vx_k;T_S,T_B)=1$ if and only if the candidate trajectory remains in $\calC$ over the entire horizon $[0,T_S+T_B]$ and reaches $\calS_0$ at time $T_S+T_B$. \footnote{When the backup horizon $T_B$ is fixed in the discussion, we write $\operatorname{Valid}(\vx_k;T_S)$ for notational simplicity.}
\end{definition}

For switching time $T_S$, define the validity set
\begin{equation}
\calV(T_S;T_B)
\coloneqq
\left\{
\vx_k \in \calC \ \middle| \
\operatorname{Valid}(\vx_k;T_S,T_B)=1
\right\}.
\label{eq:valid_set}
\end{equation}
When $T_B$ is fixed, we simply write $\calV(T_S)$.

This certificate is local to the candidate trajectory being checked. Specifically, MPS verifies whether the nominal-then-backup trajectory constructed $\vx_k$ remains in $\calC$ and reaches $\calS_0$; it does not require the stronger global statement that the backup policy can recover every state in $\calC$.

\begin{algorithm}[t]
    \footnotesize
    \SetAlgoLined
    \caption{\texttt{MPS}($\vx_k; \pi_{\textup{nom}}, \pi_{\textup{b}}$)}
    \label{alg:mps}
    \DontPrintSemicolon
    \KwIn{Current state~$\vx_k = \vx(t_k)$; nominal policy~$\pi_{\textup{nom}}$; backup policy~$\pi_{\textup{b}}$}
    \KwOut{Committed control input $\vu(t_k)$}
    \BlankLine
    $T_{S}^{\star} \gets 0$\;
    % Red color for the fixed check
    \color{green}
    $T_{S} \gets \Delta t$\;
    \color{black}
    Construct $\chi^{\textup{can}}(\cdot;\vx_k,T_S)$ using $\pi_{\textup{nom}}$ over $[0,T_{S}]$ and $\pi_{\textup{b}}$ thereafter (\autoref{def:candidate_trajectory})\;
    \If{$\chi^{\textup{can}}(\cdot;\vx_k,T_S)$ is valid by \autoref{def:valid}}{
        $T_{S}^{\star} \gets T_{S}$\;
    }
    
    \eIf{$T_{S}^{\star} > 0$}{
        \Return $\pi_{\textup{nom}}(\vx_k)$ \tcp*{Nominal policy}
    }
    {
        \Return $\pi_{\textup{b}}(\vx_k)$ \tcp*{Backup policy}
    }
\end{algorithm}

\subsection{MPS decision rule}

At update time $t_k$, MPS considers only one switching time, namely the next digital update interval,
\begin{equation}
T_S = \Delta t.
\end{equation}
Accordingly, the induced safety filter, with dependence on $\pi_{\textup{nom}}, \pi_{\textup{b}}, \calC, \calS_0, \Delta t$, and $T_B$ suppressed for brevity, is
\begin{equation}
\pi_{\textup{MPS}}(\vx_k)=
\begin{cases}
\pi_{\textup{nom}}(\vx_k), & \operatorname{Valid}(\vx_k;\Delta t)=1, \\
\pi_{\textup{b}}(\vx_k), & \operatorname{Valid}(\vx_k;\Delta t)=0.
\end{cases}
\label{eq:mps_decision_rule}
\end{equation}
Hence, the MPS filter-inactive set is
\begin{equation}
\calI_{\textup{MPS}}=\left\{\vx_k \in \calC \mid \operatorname{Valid}(\vx_k;\Delta t)=1\right\} = \calV(\Delta t).
\label{eq:mps_inactive_set}
\end{equation}

This restatement matches the logic of the original \texttt{IsRecoverable}$(\cdot)$ check in~\cite{bastani_safe_2021}, but expresses it in the same notation used throughout this paper. The essential structural point is that MPS has a singleton switching-time search set: it checks only $T_S=\Delta t$, and does not ask whether switching later than $\Delta t$ might still be safe (see \autoref{alg:mps}).

We refer interested readers to \cite{bastani_safe_2021, bastani_safe_2021a} for the original MPS formulation, and to \cite{kim_safe_2026} for a recent MPS-inspired shielding application in model-based diffusion planning.

\begin{algorithm}[t]
    \footnotesize
    \SetAlgoLined
    \caption{\texttt{\gatekeeper{}}($\vx_k; \pi_{\textup{nom}}, \pi_{\textup{b}}$)}
    \label{alg:gatekeeper}
    \DontPrintSemicolon
    \KwIn{Current state~$\vx_k = \vx(t_k)$; nominal policy~$\pi_{\textup{nom}}$; backup policy~$\pi_{\textup{b}}$}
    \KwOut{Committed control input $\vu(t_k)$}
    \BlankLine
    
    \color{wine}
    \eIf{UpdateTriggered($t_k$)}{
        \color{black}
        $T_{S}^{\star} \gets 0$\;
        \color{wine}
        \For{$T_{S} \gets T_{H}$ \KwTo $0$}{ 
        \tcp*[h]{Parallelizable} \\
            \color{black}
            Construct $\chi^{\textup{can}}(\cdot;\vx_k,T_S)$ using $\pi_{\textup{nom}}$ over $[0,T_{S}]$ and $\pi_{\textup{b}}$ thereafter (\autoref{def:candidate_trajectory})\;
            \If{$\chi^{\textup{can}}(\cdot;\vx_k,T_S)$ is valid by \autoref{def:valid}}{
                $T_{S}^{\star} \gets T_{S}$\;
                \textbf{break}\;
            }
        }
    }
    {
        $T_S^\star \gets \max\{T_S^\star-\Delta t,\,0\}$ \;
    }
    \color{black}
    
    \BlankLine
    \eIf{$T_{S}^{\star} > 0$}{
        \Return $\pi_{\textup{nom}}(\vx_k)$ \tcp*{Nominal policy}
    }
    {
        \Return $\pi_{\textup{b}}(\vx_k)$ \tcp*{Backup policy}
    }
\end{algorithm}

\section{REVIEW OF GATEKEEPER \label{sec:gatekeeper}}

\gatekeeper{}~\cite{agrawal_gatekeeper_2024} $\pi_{\textup{GK}}$
uses the same backup policy, candidate trajectory, and validity check
as MPS, but treats the backup switching time $T_S$ as a decision
variable rather than fixing it a priori.

\subsection{The problem of ``safety evaluation on backup" \label{subsec:evaluation_on_backup}}

A common source of conservatism in backup-based safety filters is that safety is evaluated along the backup trajectory, rather than along the nominal trajectory itself. As a result, the filter may modify the nominal control input $\vu_{\textup{nom}}$ even when continued execution of the nominal policy would in fact remain safe.

In Backup CBF, this dependence appears explicitly in the backup-flow terms~\cite{chen_backup_2021} of the constraints \eqref{eq:backup_cbf_qp_constraint_path}-\eqref{eq:backup_cbf_qp_constraint_terminal}. In MPS, it appears through the validity check (\autoref{def:valid}) applied to the candidate trajectory after a fixed nominal execution time $\Delta t$. At update time $t_k$, MPS checks only whether $\operatorname{Valid}(\vx_k;\Delta t)=1$, i.e., whether following the nominal policy for one update interval and then switching to the backup policy remains in $\calC$ and reaches $\calS_0$. If this test fails, MPS suppresses the nominal policy and applies the backup policy instead.

This can be conservative when backup becomes feasible only after additional nominal progress. A state may fail the validity test for an early switch to backup even though continued nominal execution for a longer duration would move the system to a state from which the backup maneuver is valid. The key limitation is therefore not the use of a backup policy itself, but fixing the switching time a priori. \gatekeeper{} reduces this conservatism by searching over multiple switching times instead of fixing the switch at $\Delta t$.

\subsection{Maximizing the switching time}

Let $T_H>0$ denote the search horizon for the nominal segment. At an update state $\vx_k$, \gatekeeper{} computes
\begin{equation}
T_S^{\star}(\vx_k)
\coloneqq
\sup \{ T_S \in [0,T_H] \mid \operatorname{Valid}(\vx_k;T_S)=1 \}.
\label{eq:max_ts}
\end{equation}
Thus, $T_S^{\star}(\vx_k)$ is the longest certified duration for which the nominal policy can continue before switching to the backup policy. The corresponding safety filter, with dependence on $\pi_{\textup{nom}}, \pi_{\textup{b}}, \calC, \calS_0, T_B$, and $T_H$ suppressed for brevity, is
\begin{equation}
\pi_{\textup{GK}}(\vx_k)=
\begin{cases}
\pi_{\textup{nom}}(\vx_k), & T_S^{\star}(\vx_k) > 0, \\
\pi_{\textup{b}}(\vx_k), & T_S^{\star}(\vx_k) = 0.
\end{cases}
\label{eq:gk_decision_rule}
\end{equation}
Hence, \gatekeeper{} leaves the nominal input unchanged on the set
\begin{equation}
\calI_{\textup{GK}}
=
\left\{\vx_k \in \calC \,\middle|\, T_S^\star(\vx_k)>0\right\}
=
\quad \bigcup_{\mathclap{T_S \in (0,T_H]}} \quad
\calV(T_S;T_B).
\label{eq:gk_inactive_set}
\end{equation}
Thus, \gatekeeper{} accepts the nominal controller whenever there exists some positive certified duration of nominal execution before switching to backup.

\subsection{Update times and search strategy}

The decision rule in \eqref{eq:gk_decision_rule} describes the behavior at an update instant. The implementation itself is stateful: once a positive switching time has been certified, \gatekeeper{} may continue executing the nominal policy over subsequent monitor updates without recomputing the full search, until that certificate must be refreshed. In \autoref{alg:gatekeeper}, this is represented by storing the previously certified value of $T_S^\star$ and decrementing it by $\Delta t$ between successive updates when no new search is triggered.

In practice, the continuous search in \eqref{eq:max_ts} is approximated on the grid $\mathcal{T}_H \coloneqq \{0,\Delta t,2\Delta t,\dots,N_H\Delta t\}, \, T_H = N_H\Delta t$, and the implementation selects
\begin{equation}
T_{S}^\star(\vx_k)
\coloneqq
\max\{T_S \in \mathcal{T}_H \mid \operatorname{Valid}(\vx_k;T_S)=1\}.
\label{eq:grid_search}
\end{equation}

The original implementation~\cite{agrawal_gatekeeper_2024} in \autoref{alg:gatekeeper} performs this search sequentially, starting from the largest candidate and terminating once a valid switching time is found. The decision logic is, however, parallelizable, since for a fixed current state $\vx_k$, the values $\operatorname{Valid}(\vx_k;T_S)$ for different $T_S \in \mathcal{T}_H$ can be evaluated independently. Therefore, we propose to evaluate all candidate trajectories in parallel and then choose the largest valid switching time in the discretized set, without changing the underlying algorithmic logic.

We refer interested readers to \cite{agrawal_gatekeeper_2023, agrawal_gatekeeper_2024} for the original \gatekeeper{} formulation and \cite{agrawal_online_2025} for further theoretical discussion, and to~\cite{kim_visibilityaware_2025, naveed_provably_2026, cherenson_autonomy_2026} for additional applications in broader robotic settings.

\section{THEORETICAL ANALYSIS \label{sec:analysis}}

We compare the three filters through their filter-inactive sets, where each method leaves the nominal policy unchanged. This viewpoint is directly aligned with the principle of minimal intervention. For MPS and \gatekeeper{}, the comparison is made at update instants at which a new validity check is performed from the current state.

\subsection{MPS is a special case of \gatekeeper{}}

The relation between MPS and \gatekeeper{} is exact because both methods use the same candidate trajectory and the same validity indicator; they differ only in the set of switching times they examine.

\begin{theorem}\label{thm:mps_special_case}
Assume that $0<\Delta t \le T_H$. Then
\begin{equation}
\mathcal{I}_{\textup{MPS}} \subseteq \mathcal{I}_{\textup{GK}}.
\label{eq:mps_subset_gk}
\end{equation}
\end{theorem}

\begin{proof}
Take any $\vx_k \in \mathcal{I}_{\textup{MPS}}$. By \eqref{eq:mps_inactive_set}, this is equivalent to $\operatorname{Valid}(\vx_k;\Delta t)=1$. Since $\Delta t \in (0,T_H]$, the switching time $T_S=\Delta t$ is feasible for the \gatekeeper{} search in \eqref{eq:max_ts}. Therefore $T_S^\star(\vx_k) \ge \Delta t > 0$, which implies $\vx_k \in \mathcal{I}_{\textup{GK}}$ by \eqref{eq:gk_inactive_set}.
\end{proof}

\autoref{thm:mps_special_case} gives the precise sense in which \gatekeeper{} is less intervening than MPS under identical safety assumptions: every state accepted by MPS is also accepted by \gatekeeper{}, and \gatekeeper{} may additionally accept states that become recoverable only after a later switch to the backup policy.

\subsection{Relation with Backup CBF}

We next analyze the relation between Backup CBF and \gatekeeper{}. Unlike MPS, the connection is indirect: Backup CBF does not verify nominal-then-backup switched trajectories online, but instead renders the implicit safe set $\calS=\calR^{\pi_{\textup{b}}}(T_B;\calS_0)$ forward invariant. The relevant comparison is with states at which Backup CBF is
inactive throughout a neighborhood relative to $\calS$. We denote by $\operatorname{int}_{\calS}(\calI_{\textup{BCBF}})$ the relative interior of $\calI_{\textup{BCBF}}$ in $\calS$, i.e.,
\begin{equation}
\operatorname{int}_{\calS}(\calI_{\textup{BCBF}}) \coloneqq \left\{ \vx \in \calS \mid \exists \varepsilon>0 \text{ s.t. } B_{\varepsilon}(\vx)\cap\calS \subseteq \calI_{\textup{BCBF}} \right\},
\label{eq:relative_interior_bcbf}
\end{equation}
where $B_{\varepsilon}(\vx)\coloneqq\{\vz\in\Rn \mid \|\vz-\vx\|_2<\varepsilon\}$. We compare \gatekeeper{} with this set.

\begin{theorem}\label{thm:backup_cbf_interior}
Let $\operatorname{int}_{\calS}(\calI_{\textup{BCBF}})$ be as defined in \eqref{eq:relative_interior_bcbf}. Then,
\begin{equation}
\operatorname{int}_{\calS}(\calI_{\textup{BCBF}}) \subseteq \calI_{\textup{GK}}.
\end{equation}
\end{theorem}
Figure~\ref{fig:proof} provides a geometric illustration of the local argument used in the proof.

\begin{proof}
Let $\vx_k \in \operatorname{int}_{\calS}(\calI_{\textup{BCBF}})$ be arbitrary. We show that $\vx_k \in \calI_{\textup{GK}}$.

By \eqref{eq:relative_interior_bcbf}, there exists $\varepsilon>0$ such that $B_{\varepsilon}(\vx_k) \cap \calS \subseteq \calI_{\textup{BCBF}}$. Since Backup CBF renders $\calS$ forward invariant and $\vx_k\in\calS$, we have $\varphi_{\tau}^{\pi_{\textup{BCBF}}}(\vx_k)\in\calS$ for all $\tau\ge0$. Continuity of $\varphi_{\tau}^{\pi_{\textup{BCBF}}}(\vx_k)$ with respect to $\tau$ and the fact that $\varphi_{0}^{\pi_{\textup{BCBF}}}(\vx_k)=\vx_k$ imply that there exists $\overline T>0$ such that $\varphi_{\tau}^{\pi_{\textup{BCBF}}}(\vx_k)\in B_\varepsilon(\vx_k)$ for all $\tau\in[0,\overline T]$. Hence
\begin{equation}
\varphi_{\tau}^{\pi_{\textup{BCBF}}}(\vx_k)\in B_{\varepsilon}(\vx_k)\cap\calS \subseteq \calI_{\textup{BCBF}}, \quad \forall \tau\in[0,\overline T].
\label{eq:backupcbf_stays_inactive}
\end{equation}

By definition of $\calI_{\textup{BCBF}}$, Backup CBF is inactive at every state along this trajectory segment. Since $\pi_{\textup{nom}}$ is locally Lipschitz, the nominal closed-loop solution is unique, and thus
\begin{equation} \label{eq:backupcbf_equal_to_nominal}
\varphi_{\tau}^{\pi_{\textup{BCBF}}}(\vx_k) = \varphi_{\tau}^{\pi_{\textup{nom}}}(\vx_k), \qquad \forall \tau\in[0,\overline T].
\end{equation}
Because $\varphi_{\tau}^{\pi_{\textup{BCBF}}}(\vx_k)\in\calS$ for all $\tau\ge0$, \eqref{eq:backupcbf_equal_to_nominal} yields
\begin{equation}
\varphi_{\tau}^{\pi_{\textup{nom}}}(\vx_k)\in \calS \subseteq \calC,
\qquad \forall \tau\in[0,\overline T].
\label{eq:nominal_in_s}
\end{equation}

Choose any $\overline T_S \in (0,\min\{\overline T,T_H\}]$ and define $\vx_S \coloneqq \varphi_{\overline T_S}^{\pi_{\textup{nom}}}(\vx_k)$. By \eqref{eq:nominal_in_s}, we have $\vx_S\in\calS=\calR^{\pi_{\textup{b}}}(T_B;\calS_0)$. Hence, by the definition of the recoverable set, applying $\pi_{\textup{b}}$ from $\vx_S$ keeps the state in $\calC$ over $[0,T_B]$ and reaches $\calS_0$ at time $T_B$. Together with \eqref{eq:nominal_in_s}, this shows that the candidate trajectory that follows $\pi_{\textup{nom}}$ for $\overline T_S$ and then $\pi_{\textup{b}}$ for $T_B$ is valid, i.e., $\operatorname{Valid}(\vx_k;\overline T_S,T_B)=1$. Therefore $\vx_k\in \calV(\overline T_S;T_B)\subseteq \calI_{\textup{GK}}$ by \eqref{eq:gk_inactive_set}. Since $\vx_k$ was arbitrary, it follows that $\operatorname{int}_{\calS}(\calI_{\textup{BCBF}})\subseteq \calI_{\textup{GK}}$.
\end{proof}

\begin{figure}[t]
    \centering
\includegraphics[width=0.9\linewidth]{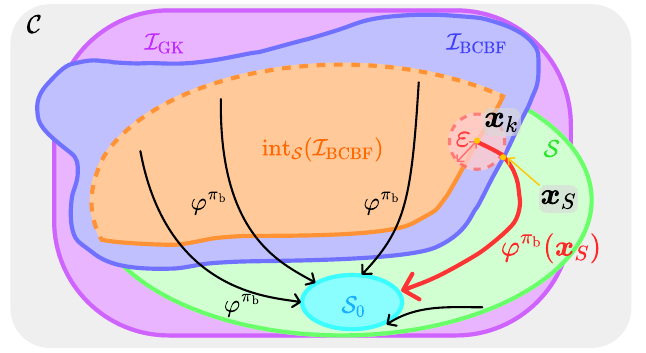}
\caption{Geometric illustration of the local argument in \autoref{thm:backup_cbf_interior}.}
\label{fig:proof}
\vspace{-5pt}
\end{figure}

\autoref{thm:backup_cbf_interior} identifies the region on which Backup CBF and \gatekeeper{} agree on nominal acceptance. Every interior point of the Backup CBF inactive set admits a nontrivial nominal segment that remains inside $\calS$, and this is sufficient for \gatekeeper{} to certify a valid switching time. Note that boundary points of $\calI_{\textup{BCBF}}$ need not admit such a positive nominal hold time, so a full inclusion $\calI_{\textup{BCBF}} \subseteq \calI_{\textup{GK}}$ cannot be claimed in general.

\subsection{Limitations}

While \gatekeeper{} reduces the conservatism associated with a fixed switching time, it does not eliminate the fundamental limitations of backup-based safety filtering.
\begin{itemize}
    \item \textbf{Dependence on horizon and search grid.} In the implemented algorithm, if $T_H$ is too short or if the grid $\mathcal{T}_H$ is too coarse, \gatekeeper{} can collapse to behavior that is nearly indistinguishable from MPS. Increasing $T_H$ or refining the grid improves permissiveness but increases computational effort.    
    
    \item \textbf{Dependence on the backup policy and model.} All three methods inherit their 
    conservatism from the quality of the backup policy and from the accuracy of the predictive model. A poor or overly aggressive backup policy shrinks the recoverable region and can mask the true advantage of a less intrusive switching rule.
\end{itemize}

\section{RESULTS}
We report three simulated studies. The first visualizes set-wise behavior in a low-dimensional example, whereas the latter two evaluate task-level performance in dynamic environments. For the reach-avoid and highway scenarios, Table~\ref{tab:merged_results} reports the fraction of time the nominal controller is preserved ($\pi_{\textup{nom}}\%$), the average certified switching time $T_S^\star$, whether the goal is reached, and the average online computation time. The safety-monitor update interval is $\Delta t=0.10~\textup{s}$
for reach-avoid and $\Delta t=0.05~\textup{s}$ for highway overtake. Additional implementation details including physical parameters can be found in our public code repository. 

\subsection{Planar Double Integrator}

We first study a planar double-integrator example. The state is $\vx=[x\ y\ v_x\ v_y]^\top\in\mathbb{R}^4$ and the input is $\vu=[a_x\ a_y]^\top\in[-a_{\max},a_{\max}]^2$, where $a_{\max}=0.5~\textup{m/s}^2$. The unsafe set is a disk centered at the origin. We compute the viability kernel using Hamilton-Jacobi (HJ) reachability analysis~\cite{bansal_hamiltonjacobi_2017}, and visualize it on the velocity slice $(v_x,v_y)=(2,0)$. The nominal policy is a PD controller that tracks the reference line $y=2$, whereas the backup policy is a PD controller that drives the system to the recovery line $y=-2$ and remains there.

Fig.~\ref{fig:double_integrator_sets} shows, for each method, the filter-inactive set together with the corresponding certified safe set. The figure makes explicit the role of the backup maneuver. Because Backup CBF and MPS certify safety through an immediate or near-immediate commitment to the backup trajectory, they reject many states from which continuing the nominal policy for a while and switching later would still be safe. \gatekeeper{} recovers these states by searching over the switching time, yielding a visibly larger nominal-acceptance region. The inclusion $\calI_{\textup{MPS}}\subseteq \calI_{\textup{GK}}$ from \autoref{thm:mps_special_case} is also clearly reflected in this example.

\begin{figure*}[t]
    \centering
    \includegraphics[width=0.87\textwidth]{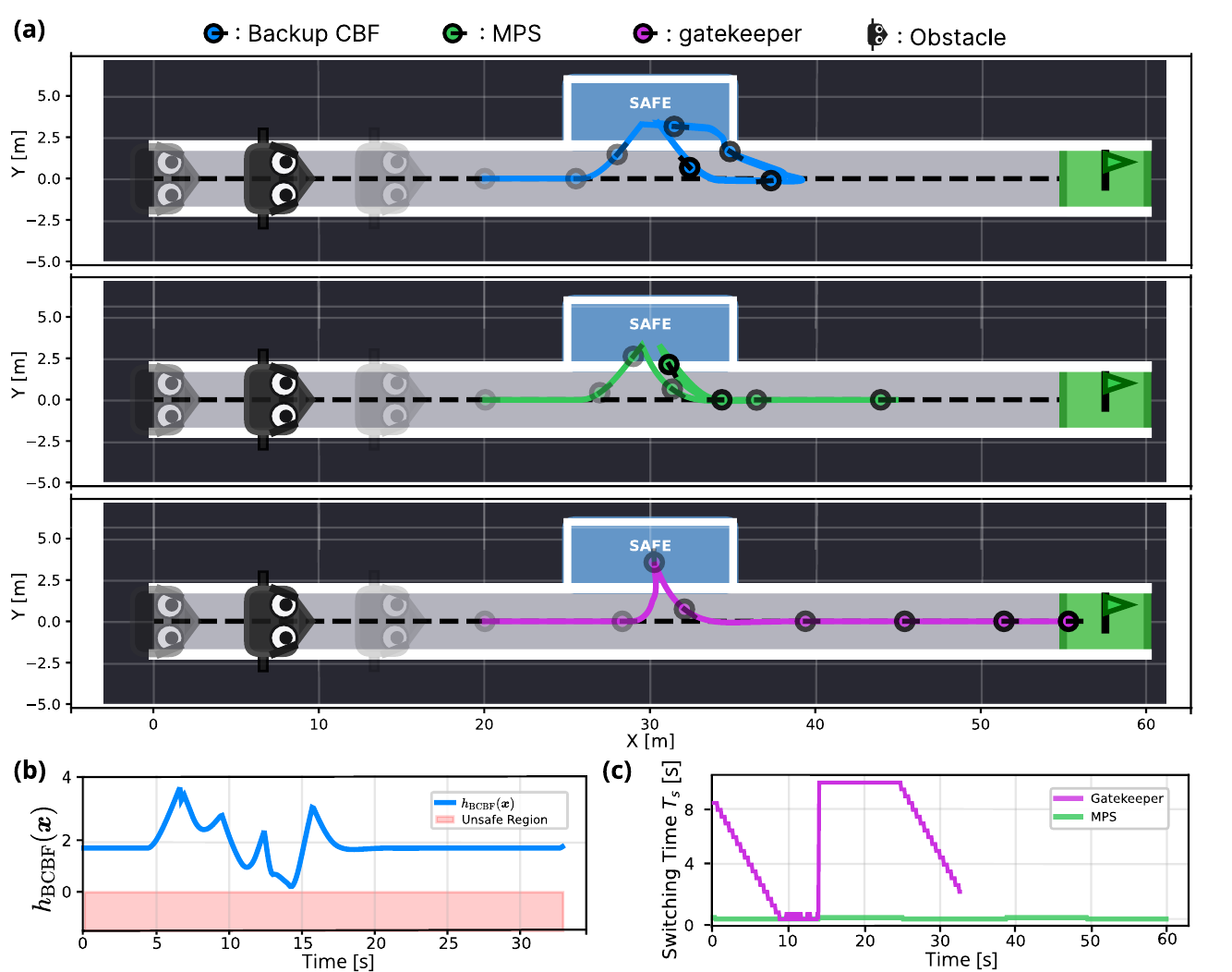}
    \caption{Reach-avoid scenario with a dynamic obstacle. (a) Robot trajectories generated by Backup CBF, MPS, and \gatekeeper{}. (b) Backup CBF value as a function of time. (c) Comparison of the certified switching time $T_S^\star$ for MPS and \gatekeeper{}.}
    \label{fig:evade_overview}
    \vspace{-5pt}
\end{figure*}

\subsection{Reach-Avoid with a Dynamic Obstacle}

We next consider a reach-avoid task in a narrow hallway with a side safety pocket. The ego robot again follows the planar double-integrator dynamics, now with radius $0.5~\textup{m}$, speed bound $\sqrt{v_x^2+v_y^2}\le 1.5~\textup{m/s}$, and acceleration bound $\|\vu\|_2\le 2.0~\textup{m/s}^2$. A moving obstacle travels through the hallway at $3.0~\textup{m/s}$, faster than the ego robot, covers the full corridor width, and respawns from the left after reaching the far end. The safety pocket and the goal region define the terminal controlled invariant set~$\calS_0$. \gatekeeper{} searches over a nominal horizon $T_H=10.0~\textup{s}$, and the backup horizon is set to $T_B=12.0~\textup{s}$ for all methods.

The nominal policy is a PD controller that drives the robot toward the goal along the hallway centerline. The backup policy is a PD controller that drives to the safety pocket and remains there; if the robot is already in the goal region, it instead keeps the robot near the goal. The representative trajectories in Fig.~\ref{fig:evade_overview} illustrate the conservatism discussed in \autoref{subsec:evaluation_on_backup}. Backup CBF and MPS repeatedly retreat to the pocket whenever the obstacle approaches, which preserves safety but prevents progress to the goal. In contrast, \gatekeeper{} hides in the pocket for the first pass, then leaves once a later switch to backup remains valid, and reaches the goal. The same trend appears in Table~\ref{tab:merged_results}: \gatekeeper{} preserves the nominal controller on $87.5\%$ of the steps and reaches the goal, whereas MPS and Backup CBF preserve the nominal controller on only $55.3\%$ and $28.8\%$, respectively, and fail to complete the task. The parallelized implementation of \gatekeeper{} yields the same closed-loop behavior with significantly lower average computation time than the original \gatekeeper{}.

\begin{table}[t]
\centering
\caption{Quantitative evaluation in the reach-avoid and highway overtake scenarios, averaged over five trials with randomly sampled initial robot conditions.}
\label{tab:merged_results}
\footnotesize
\setlength{\tabcolsep}{1.5pt}
\renewcommand{\arraystretch}{1.05}

\begin{tabular}{@{}
>{\centering\arraybackslash}m{1.45cm}
>{\centering\arraybackslash}m{1.85cm}
>{\centering\arraybackslash}m{0.85cm}
>{\centering\arraybackslash}m{1.00cm}
>{\centering\arraybackslash}m{1.05cm}
>{\centering\arraybackslash}m{1.25cm}
@{}}
\toprule
\makecell[c]{\textbf{Scenario}} &
\makecell[c]{\textbf{Method}} &
\makecell[c]{$\pi_{\textup{nom}}$ \\ \%} &
\makecell[c]{\textbf{Avg.} \\ \textbf{$T_S^\star$ [s]}} &
\makecell[c]{\textbf{Reached} \\ \textbf{Goal}} &
\makecell[c]{\textbf{Avg. Time} \\ \textbf{[ms]}} \\
\midrule
\multirow{4}{*}{\makecell[c]{Reach-Avoid}} & $\pi_{\textup{BCBF}}$ & 28.8 & N/A & \xmark & 10.33 \\
& $\pi_{\textup{MPS}}$ & 55.3 & 0.06 & \xmark & 1.31 \\
& $\pi_{\textup{GK}}$ & \textbf{87.5} & \textbf{6.00} & \checkmark & 10.86 \\
& Parallel $\pi_{\textup{GK}}$ & \textbf{87.5} & \textbf{6.00} & \checkmark & \textbf{1.30} \\
\midrule
\multirow{4}{*}{\makecell[c]{Highway\\Overtake}} & $\pi_{\textup{BCBF}}$ & 60.8 & N/A & \checkmark & 27.90 \\
& $\pi_{\textup{MPS}}$ & 67.5 & 0.03 & \checkmark & \textbf{2.64} \\
& $\pi_{\textup{GK}}$ & \textbf{100} & \textbf{6.00} & \checkmark & 3.63 \\
& Parallel $\pi_{\textup{GK}}$ & \textbf{100} & \textbf{6.00} & \checkmark & 2.97 \\
\bottomrule
\end{tabular}
\end{table}

\subsection{Highway Overtake with a Dynamic Bicycle Model}

Our final study uses a realistic highway-driving model. The ego vehicle is described by an 8-state, 2-input dynamic bicycle model with Fiala tire forces~\cite{dallas_control_2025},
\begin{equation}
\vx=[p_x\ p_y\ \psi\ r\ \beta\ V\ \delta\ \tau]^\top,\qquad
\vu=[\dot{\delta}\ \dot{\tau}]^\top,
\label{eq:sim_bicycle}
\end{equation}
where $(p_x,p_y)$ is the global position, $\psi$ is yaw, $r$ is yaw rate, $\beta$ is sideslip, $V$ is speed, $\delta$ is steering angle, and $\tau$ is rear-wheel torque. The kinematics are $\dot{p}_x=V\cos(\psi+\beta)$, $\dot{p}_y=V\sin(\psi+\beta)$, and $\dot{\psi}=r$, while the remaining dynamics are governed by nonlinear tire forces and friction-limited actuation. We impose $V\le 20~\textup{m/s}$, $|\delta|\le 20^\circ$, $|\dot{\delta}|\le 25^\circ/\textup{s}$, $|\tau|\le 4000~\textup{N\,m}$, and $|\dot{\tau}|\le 8000~\textup{N\,m/s}$. The experiments are conducted on a straight $300~\textup{m}$ road with five lanes of width $4~\textup{m}$. $T_H$ and $T_B$ are set to $6~\textup{s}$ and $3~\textup{s}$, respectively.

\begin{figure*}[t]
    \centering
\includegraphics[width=0.9\textwidth]{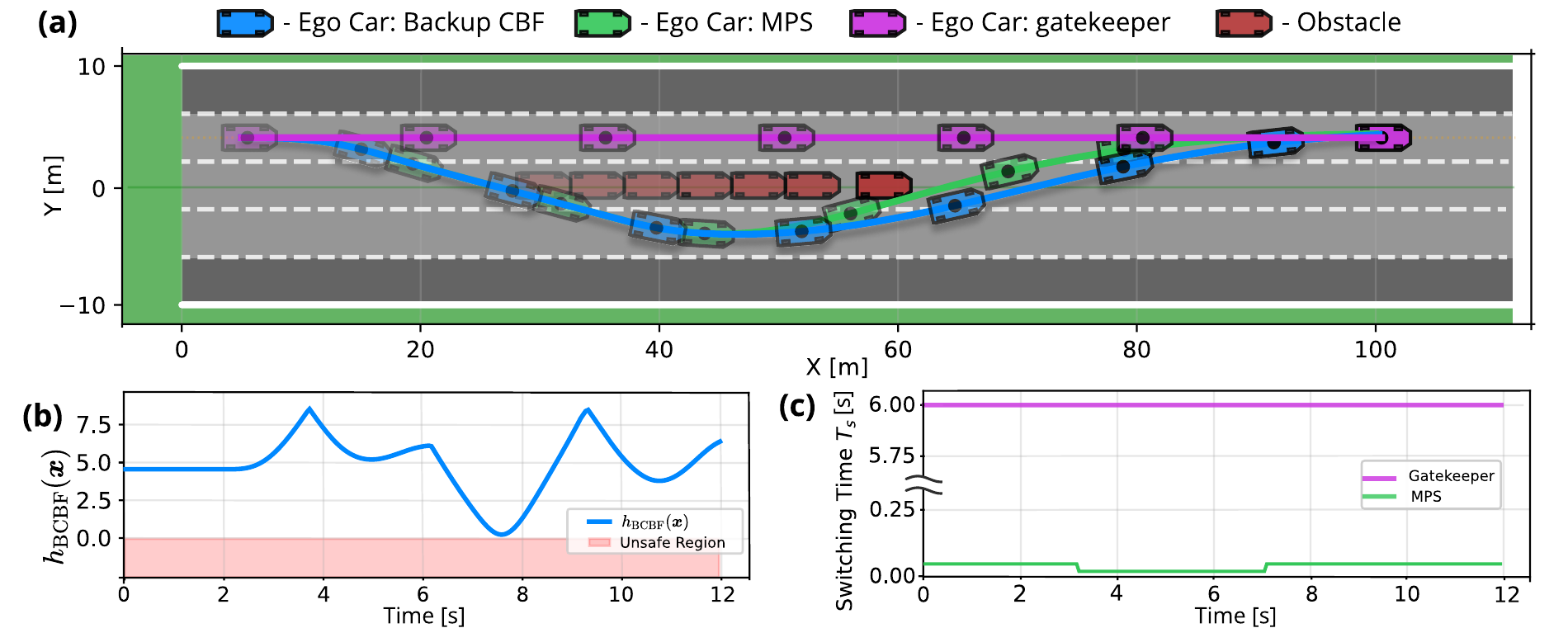}
\caption{Highway overtake scenario. (a) Vehicle trajectories generated by Backup CBF, MPS, and \gatekeeper{}. (b) Backup CBF value as a function of time. (c) Comparison of the certified switching time $T_S^\star$ for MPS and \gatekeeper{}.}
\label{fig:highway_overview}
\vspace{-5pt}
\end{figure*}

The nominal controller is a Model Predictive Contouring Controller (MPCC)~\cite{lam_model_2010} that tracks the centerline of the current lane at $10~\textup{m/s}$. The backup policy is a cascaded PD lane-change controller to a designated backup lane. In the representative case shown in Fig.~\ref{fig:highway_overview}, the nominal lane is clear, the adjacent lane contains a moving obstacle, and the backup lane farther to the side is empty. Nevertheless, Backup CBF and MPS trigger the backup lane change unnecessarily, because their online safety certificates are tied to the backup maneuver itself and evaluate the nominal trajectory only myopically. \gatekeeper{}, by certifying a longer nominal segment before switching, continues straight and avoids this unnecessary intervention. Averaged over 5 randomized trials with the same obstacle configuration, Table~\ref{tab:merged_results} shows $100\%$ nominal tracking for \gatekeeper{} versus $67.5\%$ for MPS and $60.8\%$ for Backup CBF.

\section{CONCLUSION}

This review paper revisited Backup CBF, MPS, and \gatekeeper{} under a common safety-filter abstraction. Our unified view provides a new perspective for comparing backup-based safety filters through their filter-inactive sets. It also highlights an important and still underappreciated issue shared by backup-based safety filters, which we refer to as \emph{safety evaluation on backup}: the nominal policy is evaluated myopically through the backup maneuver. Finally, we showed that \gatekeeper{} introduces an additional degree of freedom through the switching time, which can mitigate this issue and substantially increase the duration for which the nominal policy can be executed safely.

\addtolength{\textheight}{0 cm}   % This command serves to balance the column lengths

\bibliographystyle{IEEEtran}
\typeout{}
\bibliography{references.bib}

\end{document}

%% file: preamble.tex
\usepackage{graphics} % for pdf, bitmapped graphics files
\usepackage{subfig} % for subcaption
\usepackage{wrapfig}
\usepackage{amsthm} % for definition, theorm and etc
\usepackage{amsmath,amssymb}
\usepackage{graphicx}
\usepackage{tabularx}
\usepackage{multirow, makecell}
\usepackage{diagbox}
\usepackage{slashbox}
\usepackage{rotating}
\usepackage{cite}

\usepackage{url}
\usepackage{bm}
\usepackage[table]{xcolor}
\usepackage{hyperref}
\usepackage{siunitx} % degree symbol
\usepackage{booktabs}
\usepackage{cleveref}
\usepackage{mathtools} % for := \coloneqq , and split equation with raisetag. ex means x-height
\usepackage{upgreek} % to use uptau

\let\labelindent\relax % not including labelindent (avoid error for overloading \proof)
\usepackage{enumitem} % using enumerate with ref roman option

\usepackage[ruled,vlined]{algorithm2e}

\usepackage{pifont} % for x mark 
\newcommand{\xmark}{\ding{55}}%

\usepackage[nolist, nohyperlinks]{acronym}
\acrodef{MPC}[MPC]{Model Predictive Control}
\acrodef{QP}[QP]{Quadratic Program}
\acrodef{CBF}[CBF]{Control Barrier Function}

\newcommand{\vx}{{\boldsymbol x}}
\newcommand{\vu}{{\boldsymbol u}}

\newcommand{\vz}{{\boldsymbol z}}

\newcommand{\StateSpace}{\mathcal{X}}
\newcommand{\ControlSpace}{\mathcal{U}}
\newcommand{\RealSpace}{\mathbb{R}}

\newcommand{\Rn}{\mathbb{R}^{n}}
\newcommand{\Rm}{\mathbb{R}^{m}}
\newcommand{\Rnm}{\mathbb{R}^{n \times m}}

\newcommand{\calC}{\mathcal{C}} % invariant set

\newcommand{\calV}{\mathcal{V}} % vertice
\newcommand{\calS}{\mathcal{S}} % safe set
\newcommand{\calA}{\mathcal{A}} % locally validated parameters
\newcommand{\calR}{\mathcal{R}} % recovery set
\newcommand{\calI}{\mathcal{I}} % inactive set

\DeclareMathOperator*{\argmin}{arg\,min}

\newtheorem{definition}{Definition}
\newtheorem{theorem}{Theorem}
\theoremstyle{definition}
\theoremstyle{definition}
\newtheorem{problem}{Problem}

\theoremstyle{definition}

\theoremstyle{definition}

\usepackage{fontawesome5}

\newcommand{\gatekeeper}{\texttt{gatekeeper}}

\definecolor{wine}{RGB}{204, 0, 102}
\definecolor{magenta_wine}{RGB}{158, 44, 143}
\definecolor{dusty_wine}{RGB}{143, 59, 101}
\definecolor{ocean}{RGB}{13, 121, 202}
\definecolor{light_ocean}{RGB}{18, 178, 235}
\definecolor{dark_ocean}{RGB}{10, 89, 148}
\definecolor{grey}{RGB}{170, 170, 170}
\definecolor{light-grey}{RGB}{220, 220, 220}
\definecolor{dark_gray}{rgb}{0.2, 0.2, 0.2} 
\definecolor{med-grey}{rgb}{0.3, 0.3, 0.3} 
\definecolor{grape}{RGB}{112,48,160}
\definecolor{aqua}{RGB}{52,172,139}
\definecolor{dark_aqua}{RGB}{35,115,93}
\definecolor{dark_orange}{RGB}{216,92,0}
\definecolor{vibrant_orange}{RGB}{250, 160, 26}
\definecolor{vibrant_blue}{RGB}{14, 120, 255}
\definecolor{vibrant_pink}{RGB}{255, 0, 104}
\definecolor{dark_red}{RGB}{122, 0, 0}
\definecolor{dark_green}{RGB}{0, 92, 34}
\definecolor{dusty_blue}{RGB}{77, 91, 128}
\definecolor{dark_brown}{RGB}{125, 54, 36}
\definecolor{violet}{RGB}{116, 12, 173}